\documentclass{article}
\usepackage{ijcai17}
\usepackage{amsmath}
\usepackage{amsfonts}
\usepackage{amssymb}
\usepackage{amsthm}

\usepackage{graphicx}
\usepackage{listings}
\usepackage{algorithm}
\usepackage[noend]{algorithmic}
\usepackage{subfigure}
\usepackage{wrapfig}
\usepackage{verbatim}
\usepackage[T1]{fontenc}
\usepackage{times}
\usepackage[small]{caption}

\DeclareMathOperator*{\argmax}{\arg\!\max}

\newtheorem{theorem}{Theorem}

\title{Heuristic Online Goal Recognition in Continuous Domains} 
\author{Mor Vered\\ 
Bar Ilan University, Ramat-Gan, Israel\\
veredm@cs.biu.ac.il
\And
Gal A. Kaminka\\
Bar Ilan University, Ramat-Gan, Israel\\
galk@cs.biu.ac.il}

\begin{document}

\maketitle

\begin{abstract}
Goal recognition is the problem of inferring the goal of an agent, based
on its observed actions. 
An inspiring approach---plan recognition by planning (PRP)---uses 
off-the-shelf planners to dynamically generate plans for given goals, eliminating the need for the traditional plan library.
However, existing PRP formulation is inherently inefficient in online recognition, 
and cannot be used with motion planners for continuous spaces.
In this paper, we utilize a different PRP formulation which
allows for online goal recognition, and for application in continuous spaces.
We present an online recognition algorithm, where two heuristic decision points 
may be used to improve run-time significantly over existing work. 
We specify heuristics for continuous domains, prove guarantees on their use, and empirically evaluate the algorithm over n 
 hundreds of experiments in both a 3D navigational environment and a cooperative robotic team task.
\end{abstract}
\vspace{-8pt}

\sloppy

\section{Introduction}
\vspace{-4pt}

Goal recognition is the problem of inferring the (unobserved) goal of an agent, based
on a sequence of its observed actions~\cite{hong01,blaylock06,baker2007goal,lesh95}. % Journal: blaylock04,,lesh95
 It is a fundamental research problem in artificial intelligence, closely related to plan, activity, and intent recognition~\cite{pairbook}. %Journal: 
The traditional approach to plan recognition is through the use of a plan-library, a pre-calculated library of known plans to achieve known goals~\cite{pairbook}. 
Ramirez and Geffner~\shortcite{ramirez2010probabilistic}
introduced a seminal recognition approach which avoids the use of a plan library completely. Given a set of goals $G$, the \textit{Plan Recognition by Planning (PRP)} approach uses off-the-shelf planners in a blackbox fashion, to dynamically generate % journal: plans that serve as 
recognition hypotheses as needed. 

The use of PRP in continuous domains, and in an online fashion (i.e., when observations are made incrementally)
raises new challenges. The original~\shortcite{ramirez2010probabilistic} formulation, relies on synthesizing two optimal plans for every goal $g\in G$: (i)  a plan to reach goal $g$ in a manner compatible with the observations $O$; and (ii) a plan to
reach goal $g$ while (\textit{at least partially}) deviating from $O$, i.e. complying with $\overline{O}$. The likelihood of each goal is computed from the difference in \textit{costs} of optimal solutions to the two plans.
%The likelihood $Pr(g|O)$ is computed for each $g\in G$ from $\Delta(g,O)$, the difference in \textit{costs} of optimal solutions to the two plans. 
Overall, $2|G|$ planning problems are solved, two for each goal. 

However, in \emph{online} recognition the set $O$ is incrementally revealed, and $\overline{O}$ changes with it. Thus two new planning problems are  solved with \textit{every new observation}, for a total of $2|G||O|$ calls to the planner instead of $2|G|$. In addition, using an off-the-shelf continuous-space planner to generate a plan that may \emph{partially} go through previous observations, but must not go through all of them, is currently impossible given the state of the art.

We present a general heuristic algorithm for online recognition in continuous domains that solves \textit{at most} $|G|(|O|+1)$ planning problems, and \textit{at best}, $|G|$. The algorithm relies on an alternative formulation, that does not use $\overline{O}$. It has two key decision points where appropriate heuristics reduce the number of calls to the planner: for each new observation, the first decision is whether to generate and solve a new planning problem for each $g$, or remain with the former calculated plans. In the best case, this may reduce the number of overall calls to the planner to only $|G|$ calls. A second decision is whether to prune unlikely goal candidates, 
incrementally reducing $|G|$, thus making fewer calls to the planner. % for any new observations. 

We describe the algorithm in detail, and examine several heuristic variants. % for online goal recognition in continuous environments. 
Utilizing off-the-shelf continuous-space planners, without any modification, we
evaluate the different variants in hundreds of recognition problems, in two continuous-environment tasks: a standard motion planning benchmark, and simulated ROS-enabled robots utilizing goal-recognition for coordination.

\section{Related Work}
\label{sec:related}
\vspace{-4pt}

% Related Work
% Journal:
%Goal recognition has many applications, including for example human-robot and human-computer interaction~\cite{wang2013probabilistic}, %journal: lesh99,barbara plan recognition work,
% command prediction~\cite{lesh95,blaylock06} %journal blaylock04,
% intelligent learning environments~\cite{uzan15,amir13}, recognizing navigation goals~\cite{liao07b,zhu1991hidden}, etc.
 %The problem has many applications in continuous environments, e.g., for recognizing intended
 %gestures %Journal: ~\citep{rubine1991specifying}
 %and sketches~\citep{sezgin2005hmm}, or for recognizing intended navigational
 %goals~\citep{zhu1991hidden}. 
Sukthankar et al.~\shortcite{pairbook} provide a survey of recent work in goal and plan recognition, most of it
assuming a library of plans for recognition of  goals. Though successful in many applications, library-based methods
%Such methods
%include a variety of probablistic inference techniques~\cite{bui2003general,aaai07upr}, %journal: liang2015estimating, consider douglas07
% grammar-based approaches~\cite{pynadath00,geib2009probabilistic,slim16}, %journal: geib2011recognizing
%  and many more. % (e.g.,~\cite{}) %journal: aamas12bostjan,ijcai05
% The use of a plan library limits recognition to known plans
are limited to recognizing known plans. 
% If the observations are of an unknown plan, even leading to a known goal, library-based methods fail. 
 Alternative methods have been sought.

 %Also, when we wish to add to the set of recognizable goals, we need to also
 % insert plans for recognizing the new goals into the library (e.g., manually or by learning), in order for the goals to be recognized. 

%The plan library efficiently represents all known plans to achieve known goals;
% methods vary in the representation and inference algorithms used: action decomposition graphs~\citep{kautz86,ijcai05}, Bayesian networks~\citep{charniak93,albrecht97}, hidden Markov model variants~\citep{blaylock04,bui03}, %journal: jida09, many!
 %conditional  random fields~\citep{liao07b,douglas07,liang2015estimating},
 % %journa: hu08
 %grammars~\citep{pynadath00,geib2011recognizing}, %Journal: geib08,geib15
 % case-based plan recognition~\citep{vattam2015error} 
 %etc.

% Others have taken an online mirroring approach to recognition. \cite{sadeghipour2009social,sadeghipour2011embodied}  
% utilize efficient plan libraries to explicitly represent (and store) complete
% shape drawings and gestures \textit{plans}, that  % (drawn via gestures in air),
% % with fixed origin points.
% can be used both for recognition and execution by the agent. This dictates that
% different rotations of the same shape are stored separately. Analogously,~\cite{tambe1995resc,laird01} use 
% a virtual agent's own BDI plan to recognize a BDI plan being executed by another agent.
% In contrast, we do not store plans, but instead use a \textit{planner} to generate plans on the fly. 
 
Geib~\shortcite{geib15acs}, and Sadeghipour et al.~\shortcite{sadeghipour2011embodied} % journal: sadeghipour2009social,resc95,laid01
offer methods that utilize the same library for both planning and plan-recognition.
Hong~\shortcite{hong01} presents an online method, with no use of a library, but lacking the ranking of the recognized goals. Baker et.al~\shortcite{baker2005bayesian} present a Bayesian framework to calculate goal likelihoods, marginalizing over possible actions. %journal: and generating state transitions, 
%Journal: 
Keren et al.~\cite{keren15} investigate ways to ease goal recognition by modifying the domain.

Ramirez and Geffner~\shortcite{ramirez2010probabilistic} proposed the  PRP formulation (which plans for $g$ twice: with $O$ and with $\overline{O}$), for offline recognition. We build on their earlier formulation~\shortcite{ramirez2009plan}, 
in which they did not probabilistically rank the hypotheses, as we do here. This allows us to more efficiently compute the likelihood of different goals, given incrementally revealed observations. We embed this formulation in a definition of plan recognition
for continuous spaces, which also varies from the original in that the recognizer observes effects, rather than actions.

Other investigations of PRP exist. \cite{masters2017cost} provided a simpler formula than that of~\shortcite{ramirez2010probabilistic} achieving identical results in half the time, still in discrete environments. Sohrabi et. al~\shortcite{riabov2016plan} also observe effects, though in discrete environments, and have also sought to eliminate planner calls, by using a $k$-best planner in an offline manner to sample the plans explaining the observations. Ramirez and Geffner~\shortcite{ramirez2011goal} extend the model to include POMDP settings with partially observable states. 
Martin et. al~\shortcite{yolanda2015fast} and Pereira et al.~\shortcite{aaai17landmarks} refrain from using a planner at all, instead using pre-computed information (cost estimates and landmarks, resp.) to significantly speed up the recognition. These approaches complement ours.
Vered and Kaminka~\shortcite{acs16} present an online recognizer, which we prove is a special case of the algorithm
we present here. We go a significant step beyond by introducing heuristics to significantly improve both run-time and accuracy.

%Both of these approaches work in discrete domains only (e.g., using
% PDDL-capable planners), where there is no uncertainty in the observations, and observed actions are
% discretely defined. The latter may also be inefficient in online recognition,
 % where observations are incrementally revealed.
 
% Inspired by these investigations, \textit{goal mirroring} uses a planner (as a black box) to generate plan and goal hypotheses on-the-fly.
% It departs from earlier work by uniquely addressing \emph{online} recognition,
% providing an efficient algorithm for online recognition and a empirically
% evaluating the number of calls made to the planner in the process and overall
% planner run-time.
%    
%  

\section{Goal Recognition in Continuous Spaces}
\label{sec:goal-mirror}
\vspace{-4pt}

We begin by giving a general definition of the goal recognition problem in continuous spaces ( Section \ref{sec:def}). 
We proceed to develop an efficient \emph{online} recognizer, which can utilize heuristics to further improve the efficiency ( Section \ref{sec:alg}).
We then discuss such heuristics in detail ( Section \ref{sec:heuristics}).

\subsection{Problem Formulation}
\label{sec:def}

We define $R$, the \emph{online goal recognition problem in continuous spaces} as a quintuple $R=\langle W, I, G, O, M\rangle$.  
$W\subseteq\mathbb{R}^n$ is the world in which the observed motion takes place, as defined in standard motion planning~\cite{lavalle06}. 
$I \in W$, the initial pose of the agent. 
 $G$, a set of 
 goals; 
 each goal $g\in W$, i.e.,   
a point.  $O$, a discrete set of observations, where for all $o\in O$, $o\subset W$, i.e.,
each observation a specific subset of the work area i.e., a point or trajectory.
$M$, a (potentially infinite) set of \emph{plan trajectories}, each beginning in $I$,
and ending in one of the goal positions $g\in G$. For each goal $g$, there exists
at least one plan $m_g \in M$ that has it as its end point.

Intuitively, given the problem $R$, a solution to the \emph{goal
recognition} problem is a specific goal $v\in G$ that 
\emph{best matches} the observations $O$. 
For each goal $g$, trajectories $m_g$
% \in M$ 
(ending with $g$) are matched against the observations $O$.

Formally, we seek to determine $v\equiv argmax_{g\in G} Pr(g|O)$. 
Ramirez and Geffner~\shortcite[Thm. 7]{ramirez2009plan} have shown that necessarily, a goal $g$ is a solution to the goal recognition problem \emph{iff} the cost of an \textit{optimal} plan to achieve $g$ (denoted here $i_g$, for \emph{ideal} plan) is equal to the cost of an \textit{optimal} plan that achieves $g$, while including \emph{all} 
the observations (a plan we refer to as $m_g$).
Vered and Kaminka~\shortcite{acs16} build on this to establish a ranking over the goals. They define the ratio $score(g)\equiv \frac{cost(i_g)}{cost(m_g)}$, and rank goals higher as $score(g)$ gets closer to 1. They show experimentally that the ratio works well
in continuous domains, and thus we use it here (ignoring priors on $Pr(g)$ for simplicity):
$Pr(g|O)\equiv\eta\textit{score}(g)$, 
where the normalizing constant $\eta$ is $1/\sum_{g\in G}{\textit{score}(g)}$.

The next step is to compute the plans $i_g$ (ideal plan, from initial pose $I$ to goal $g$) and $m_g$ (an optimal plan that includes the observations). As described in~\cite{acs16}, computing $i_g$ is a straightforward application of a planner.
The synthesis of $m_g$ is a bit more complex, as $m_g$ candidates  must  minimize the error in matching the observations. 

To do this, we take advantage of the opportunity afforded by the equal footing of observations $O$ and plans in $M$ in continuous environments. Each observation is a trajectory or point in continuous space. Each plan is likewise a trajectory in the same space, as plans are modeled by their effects.
Thus generating a plan $m_g$ that perfectly matches the observations is done by composing it from two parts: 

\begin{itemize}
	\item A plan prefix,
	(denoted $m^-_g$) is built by concatenating all observations in $O$ into a single 
	 trajectory (~\cite{masters2017cost} have shown that the same plan prefix may be generated for all possible trajectories ). 
	\item A plan suffix (denoted $m^+_g$) is generated by calling the planner, to generate a trajectory from the last observed point in the prefix $m^-_g$ (the ending point of the last observation in $O$) to the goal $g$.  
\end{itemize}
Using $\oplus$ to denote trajectory concatenation, a plan $m_g\equiv m^-_g\oplus m^+_g$ is a trajectory from the first observed point in $O$, to $g$.  Notice that $m_g$ necessarily \textit{perfectly} matches the observations $O$, since it incorporates them.

Given a goal $g$ and a sequence of observations $O$, the planner is called twice: to generate
$i_g$ and to generate $m_g^{+}$, used to construct $m_g$. The cost of $i_g$ and $m_g$ is
contrasted using a scoring procedure, denoted $match(m_g,i_g)$, which uses the ratio as
described above.  As $i_g$ does not depend on $O$, it can be generated once for every goal,
while $m_g$ needs to be re-synthesized from its component parts as $O$ is incrementally revealed.
This establishes the baseline of $(1+|O|)|G|$ calls to the planner~\cite{acs16}. We now generalize this procedure to further improve on this baseline.

\subsection{Heuristic Online Recognition Algorithm}

\label{sec:alg}

We identify two key decision points in the baseline recognition
process described above, that can be used to improve its efficiency: 
\begin{itemize}
\item Recompute plans only if necessary, i.e., if the new observation may change the ranking (captured by a $RECOMPUTE$ function); 
\item Prune (eliminate) goals which are impossible or extremely unlikely (as they deviate too much from the ideal plan $i_g$), (captured by the $PRUNE$ function).
\end{itemize}
 A good $RECOMPUTE$ heuristic reduces calls to the planner by avoiding unnecessary computation of $m_g^+$ for new observations. A good $PRUNE$ heuristic reduces calls to the planner by eliminating goals from being considered for future observations.
Using appropriate heuristics in these functions, we can reduce the number of calls made to the planner and consequently overall recognition run-time.  This section presents the algorithm. The next section will examine candidate heuristics.

Algorithm~\ref{alg:online} begins (lines 3--5) by computing the ideal plan $i_g$ for all goals, once. It also sets $i_g$ as a default plan suffix $m_g^+$. This suffix guarantees that valid (though not necessarily optimal plans) $m_g$ can be created from $m_g^+$, even in the extreme case where no computation of $m_g^+$ is ever done.
Then, the main loop begins (line 6), iterating over observations as they are made available. 
We then reach the first decision: should we recompute the suffix $m_g^+$ (line 7).

We will begin by giving a general outline.
The $RECOMPUTE$ function takes the current winning trajectory $m_\upsilon$ ($\upsilon$ is the current top-ranked goal) and the latest observation $o$. It matches the observation to $m_\upsilon$ and heuristically determines (see next section) whether $o$ may cause a change in the ranking of the top goal $\upsilon$.  If so, then the suffixes $m_g^+$ of all goals (lines 8--12) will be recomputed (lines 11--12), unless pruned (lines 9--10). Otherwise (lines 13--15) the current suffix $m_g^+$ of all goals will be modified based on $o$, but without calling the planner.

\begin{algorithm}[htbp]
\caption{{\sc Online Goal Recognition} $(R, planner)$}
\label{alg:online}
\begin{algorithmic}[1]

\STATE $\forall g: m_g,m_g^-\leftarrow\emptyset$

\STATE$\upsilon\leftarrow\emptyset$ \hfill $\rhd$ the top-ranked goal 
\FORALL{$g\in G$}
	\STATE	$i_g \leftarrow \textit{planner}(I,g)$
	\STATE  $m_g^+ \leftarrow i_g$ \hfill $\rhd$ default value for plan suffix
\ENDFOR
\WHILE{new $o\in O$ available}

	\IF{ $\sc{RECOMPUTE}(m_\upsilon,o) $}

		\FORALL{$g\in G$}
			\IF{ $\sc{PRUNE}(m_g^+,o,g)$}

				\STATE $G \leftarrow G-\{g\}$
			\ELSE

				\STATE $m_g^+ \leftarrow \textit{planner}(o,g)$

			\ENDIF
		\ENDFOR
	\ELSE
		\FORALL{$g\in G$}
			\STATE $m_g^+ \leftarrow m_g^+ \ominus \textit{prefix}(o,m_g^+)$ %this is i_g from current point to goal, without observations
		\ENDFOR
	\ENDIF
	\FORALL{$g\in G$}
 		\STATE $m_g^- \leftarrow m_g^-\oplus o$
		\STATE $m_g\leftarrow m_g^-\oplus m_g^+$

	\ENDFOR
	\FORALL{$g\in G$}
			\STATE $Pr(g|O)\leftarrow\eta\cdot\textit{score}(g)$
	\ENDFOR
	\STATE $\upsilon\leftarrow\argmax_{g\in G}{P(g|O)}$
\ENDWHILE
\end{algorithmic}
\end{algorithm}

\paragraph{Recomputing $\mathbf{m_g^+}$.} Here a straightforward call to the planner is made per the discussion in section \ref{sec:def}, to generate an optimal trajectory. From the initial point (the \textit{last} point in $o$, as $o$ might contain more than a single point), to the goal $g$.

\paragraph{Modifying $\mathbf{m_g^+}$.} When no recomputation of the suffix is deemed necessary, $o$ will be added to the prefix $m_g^-$, and the existing $m_g^+$ must be updated so that it continues $m_g^-$ %the new prefix 
and leads towards $g$. The baseline algorithm calls a planner to do this, but the point of this step is to approximate the planner call so as to avoid its run-time cost. This is done by removing (denoted by $\ominus$) any parts that are \emph{inconsistent} with respect to the observation from the beginning of the old suffix $m_g^+$ . 

The old $m_g^+$ begins where the \emph{old} $m_g^-$ (without $o$) ended.  The new $m_g^+$ should ideally begin with the last point of the \emph{new} $m_g^-$ (which is the new observation $o$), and continue as much as possible with the old $m_g^+$. Thus a prefix of the old $m_g^+$, denoted ($prefix(o,m_g^+)$, line 15) is made redundant by $o$ and needs to be removed. If $o$ is directly on $m_g^+$, then $prefix(o,m_g^+)$ is exactly the trajectory from the beginning point of $m_g^+$ to $o$. But in general, we cannot expect $o$ to be directly on $m_g^+$. We thus define the ending point
for the prefix to be $\widetilde{o}$, the geometrically closest point to $o$ on $m_g^+$.

\paragraph{Pruning.} Intuitively, when the newest observation $o$
leads \emph{away} from a goal $g$, we may want to eliminate the goal from being considered further, by permanently removing it from $G$. This is a risky decision, as a mistake will cause the algorithm to become unsound (will not return the correct result, even given all the observations). On the other hand, a series of correct decisions here can incrementally reduce $G$ down to a singleton ($|G|=1$), which will mean that the number of calls to the planner in the best case will approximate $(|O|+1)$.

\paragraph{}
Finally, when the algorithm reaches line 16, a valid suffix $m_g^+$ is available for all goals in $G$. For all of them, it then concatenates the latest observation to the prefix $m_g^-$ (line 17), and creates a new plan $m_g$ by concatenating the prefix and suffix (line 18). This means that a new $score(g)$ can be used to estimate $Pr(g|O)$ (lines 19--20), and a (potentially new) top-ranked goal $\upsilon$ to be selected (line 21).

\subsection{Recognition Heuristics}
\label{sec:heuristics}
Algorithm~\ref{alg:online} is a generalization of the algorithm described in~\cite{acs16}. By varying the  heuristic functions used, we can specialize its behavior to be exactly the same (Thm.~\ref{thm:baseline}), or change its behavior in different ways.

\begin{theorem}
\label{thm:baseline}
If $RECOMPUTE=\top$ and $PRUNE=\bot$ then Algorithm~\ref{alg:online} will generate exactly the same number of planner calls as the algorithm in~\cite{acs16}.
\end{theorem}

\begin{proof}
({\it Sketch}) By setting $RECOMPUTE$ to be \textit{always true}, and  $PRUNE$ to be always false, initially a single call to the planner will be made to calculate $i_g$, and then a new call to generate $m_g^+$ will be made for all goals and every observation (since no calls will be skipped and no goals would be pruned). This is in accordance with the behavior of the algorithm reported in~\cite{acs16}.
\end{proof}

Let us now turn to examine how to avoid unnecessary planner calls. An ideal scenario would be for Alg.~\ref{alg:online} to never compute a new suffix $m_g^+$, for any goal. In line 5 of Alg.~\ref{alg:online} the initial suffixes $m_g^+$ are set to the ideal plans, $i_g$. If $RECOMPUTE$ is
\emph{always false}, then no new planner calls will be made and $m_g^+$ will be incrementally modified (Alg.~\ref{alg:online}, line 15) to accommodate the observations.

This approach offers significant savings (Thm.~\ref{thm:norecompute}), and in the best case, when the observations closely match the originally calculated paths, can produce good recognition results. 
 However, realistically, the observations may contain a certain amount of noise or the observed agent may not be perfectly rational. 
Moreover, it could be that the observed agent is perfectly rational and there is no noise in the observations and yet the approach will
fail. This is due to cases where there are multiple perfectly-rational (optimal) plans, which differ from each other but have the exact same optimal cost. In such cases, it is possible that the planner used by the recognizer will generate an ideal plan $i_g$ which differs
from an equivalent---but different---ideal plan $m_g$ carried out by the observed agent.

\begin{theorem}
\label{thm:norecompute}
If $RECOMPUTE=\bot$ there will be exactly $|G|$ calls to the planner.
\end{theorem}

\begin{proof} 
Straightforward, omitted for space.
\end{proof}

\vspace{-8pt}
\paragraph{$\mathbf{RECOMPUTE}$.}
As we cannot realistically expect the observations to perfectly match the predictions, we need a heuristic that evaluates to \emph{false} when the new observation $o$ does not alter the top ranked goal $\upsilon$ (saving a redundant $|G|$ calls to the planner), and evaluates to \emph{true} otherwise.  

A suggestion for such a heuristic in continuous domains is to measure the shortest distance $dist(o,m_g)$
between $o$ and all plans $m_g$. If $dist(o,m_\upsilon)$ is shortest we can assume the observed agent is still heading towards the same goal and do not
need to re-call the planner, keeping the current rankings.

\vspace{-4pt}
\paragraph{$\mathbf{PRUNE}$.}
Finally, we introduce a pruning heuristic for observing rational agents in continuous domains. It
is inspired by studies of human estimates of
intentionality and intended action~\cite{bonchek2014towards}.
Such studies have shown a strong bias on part of humans to prefer hypotheses that interpret
motions as continuing in straight lines, i.e., without deviations from, or
corrections to, the heading of movements. Once a rational agent is moving away or past a goal
point $g$, it is considered an unlikely target.

\begin{figure}[htbp]
\centering
\includegraphics[clip,width=0.60\columnwidth,keepaspectratio]{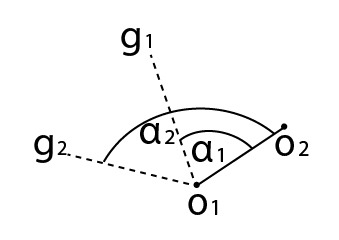}
\caption{ Illustration of goal angles used in pruning heuristic.}
\label{fig:explainHeuristic}
\end{figure}

To capture this, the $PRUNE$ heuristic takes a geometric approach.  
We calculate $\alpha_g$, the angle created between the
end-point of  the old $m_g^-$, the newly received observation $o$ and the previously
calculated plan $m_g$.  $\alpha_g$ is calculated using the \emph{cosine}
formula, $\textit{cos}(\alpha)=(\vec{u}\cdot\vec{v})/(||\vec{u}||||\vec{v}||)$, where $\vec{u}$ is the vector created by the previous and new observation and $\vec{v}$, the vector created by the previously
calculated plan and the new observation.

Figure~\ref{fig:explainHeuristic} presents an illustration of the
heuristic approach in 2D.  For the new observation, $o_2$, we measure the angle $\alpha_i$ created
by the new observation $o_2$, the previous observation $o_1$ (which ends $m_g^-$) and the previous plans  $m_{1,2}$ (shown as the dashed lines). 
If the angle is bigger than a given threshold we deduce that the
previous path is heading in the wrong direction and prune the goal. 
By defining different sized threshold angles we can relax or strengthen
the pruning process as needed.

\section{Evaluation}
\label{sec:evaluation}
\vspace{-4pt}
% Evaluation section

We empirically evaluated our online recognition approach and the suggested heuristics over hundreds of goal recognition
problems while measuring both the efficiency of the approach in terms of run-time and overall number of 
calls to the planner, and the performance of the approach in terms of convergence and correct ranking of the chosen goal. We additionally implemented our approach on simulated ROS-enabled robots while measuring the efficiency of the algorithm as compared to two separate approaches, one containing full knowledge of the observed agents' intentions and the other containing no knowledge and no reasoning mechanism.  

%We additionally wanted to 
%evaluate the effects of planner choice on overall performance.  %Journal:, on the recognition results and runtime
 %and the sensitivity of the recognition approach, by contrasting 
%results in easier and harder goal recognition problems. 

\vspace{-2pt}
\subsection{Online Goal Recognition In A 3D Navigation Domain}
\label{sec:exp:results1}

%\paragraph{3D Navigation Recognition.}
\label{sec:exp:navigation}

We implemented our approach and the proposed heuristics to recognize the goals of navigation in 3D worlds. 
We used TRRT (Transition-based Rapidly-exploring Random Trees), an off-the-shelf planner that guarantees asymptotic near-optimality 
%by preferring shorter solutions
, available as part of the Open Motion Planning Library (OMPL \cite{ompl}) along with the  OMPL \emph{cubicles} environment and
default robot. %Journal :(Figure~\ref{subfig:scene_easy}). 
Each call to the planner was given a time limit of 1 sec. The cost measure being the length of the path. For the \emph{Pruning} heuristic we used a threshold angle of $120^\circ$.

%We implemented goal mirroring to recognize the goals of navigation in 3D worlds.
%We contrasted the use of  four off-the-shelf planners (RRT*, TRRT, RRTConnect, KPIECE1), available as part of 
%the Open Motion Planning Library (OMPL~\cite{ompl}).
%We utilized these planners in the existing \emph{cubicles} environment and the
%default robot (Figure~\ref{subfig:scene_easy}) giving them
%each a time frame of 1 sec to run.
%Here, the cost measure (Algorithm~\ref{algo:onlineGM}, lines 4 and 11) is simply
%the length of the path.

We set 11 points
spread through the cubicles environments. We then generated two observed
paths from each point to all others, for \textit{a total of $110\times 2$ goal recognition problems}.
The observations were obtained by running an RRT* planner on each pair of
points, with a time limit of 5 minutes per run. RRT* was chosen because it is an optimized planner
that guarantees asymptotic optimality. The longer the run-time the more optimal the path.
% XXX Mor:  I am intentionally not discussing how we generated. I don't want to go into a discussion of optimality in observations when we don't have room to discuss in detail, and we do not have results either way.
%We used the OMPL \emph{interpolate} method to generate between between 20 and 76 observed points for each path problem.
Each problem contained between 20-76 observed points.

\vspace{-4pt}
\paragraph{Performance Measures}

We use two measures of recognition performance : (1) the time (measured 
by number of observations from the end) in which the recognizer converged 
to the correct hypothesis (including 0 if it failed). Higher values indicate earlier convergence and are therefore better; and (2) 
the number of times they ranked the correct hypothesis
at the top (i.e., rank 1), which indicates their general accuracy. The more frequently the recognizer ranked the correct hypothesis at the top, the more reliable it is, 
hence a larger value is better.

%Again higher values indicate more correct hypotheses during the recognition process. 

%Recognizers may vary in
%hree measures: (1) the time (measured 
%by number of observations from the end) in which the recognizer converged 
%to the correct hypothesis (including 0 if it failed); (2)
%he area under the curve drawn in this graph, which gives a measure of
%he false-positive response (greater area means recognizer tended to rank
%he correct hypothesis lower, farther from top); and (3) 
%the number of times they ranked the correct hypothesis
%t the top (i.e., rank 1), which indicates their general accuracy.

%Journal:
%For example in Figure~\ref{fig:shape-navigation-examples}, the recognizer
%converged to the correct results at observation 44 out of 54. When normalizing for the observation sequence length,
%to allow comparison across different recognition problems, we measure
%the normalized convergence of the TRRT recognizer at 18.5\%. Of course the earlier the convergence, 
%the better. With regards to the second measure, the amount of times the planner ranked the correct goal as
%the top hypothesis, the  recognizer
%ranked the correct goal at the top 29 times. Again when normalizing according to observation sequence length we learn that 53.7\% of the observations were correct.% This parameter gives us 
%%an overall measure of the reliability of the recognizer. 
\vspace{-4pt}
\paragraph{Efficiency Measures}

In order to evaluate the overall \emph{efficiency} of each approach we also used two separate measures: 
(1) the number of times the planner was called within the recognition process; and (2) the overall time (in sec.) spent planning. Though these two parameters are closely linked, they are not wholly dependant. 
While a reduction in overall number of calls to the planner will also necessarily result in a reduction in planner run-time, the total amount of time allowed for each planner run may vary according to the difficulty of the planning problem and therefore create considerable differences. %For instance, pruning unlikely goals may reduce runtime considerably.

%\subsection{Results}
%\label{sec:exp:results1}

\subsection{Effects of the different heuristic approaches}
\label{sec:exp:heuristic-comparison}

We ran the TRRT based
recognizer on the above-mentioned 220 problems, comparing the different approaches. %, with a time constraint of
%at most 1 sec with separate runs for each approach. 
The results are displayed in
Table~\ref{tab:allResults}, columns 1--4. %In all graphs the X axis denotes the different approaches; the first approach, 
%comparing the different approaches;
\emph{Baseline}, refers to the algorithm of Vered et al.~\shortcite{acs16}, 
%refers to the method of refraining from recomputing the ideal path to compare against with every incoming observation with the improved online
making $(O+1)|G|$ planner calls. %(Alg.~\ref{algo:onlineGMRecompute}). 
The second approach, 
\emph{No Recomp}, refers to the method of no recomputation at all, meaning the planner is only utilized once in the beginning of the process, to calculate $i_g$, the ideal path, for all of the goals. 
%All incrementally received observations will be compared against these initially calculated plans for a total of $2|G|$ calls. 
The third approach, \emph{Recompute}, measures the effect of \emph{RECOMPUTE} which aims to reduce overall number of calls to the planner 
%by measuring the difference between the previous path 
(Section~\ref{sec:alg}). 
The fourth approach, \emph{Prune}, measures the effect of \emph{PRUNE} which aims to reduce the overall number of goals by eliminating
unlikely goal candidates (Section~\ref{sec:alg}). 
And the last approach, \emph{Both}, measures the effects of utilizing a combination of both the \emph{Pruning} and \emph{Recompute Heuristics}.

\vspace{-4pt}

\begin{table*}[htbp]
\fontsize{7.5}{9}\selectfont
\centering
\begin{tabular}{|c|c|c|c|c|c|c|c|c|}
\cline{2-9}
\multicolumn{1}{l}{} & \multicolumn{4}{|c|}{\textbf{10 goals}}   & \multicolumn{4}{|c|}{\textbf{19 goals}} \\ \cline{2-9}
\multicolumn{1}{l|}{}  & \multicolumn{2}{|c|}{\textbf{Efficiency}} & \multicolumn{2}{|c|}{\textbf{Performance}} & \multicolumn{2}{|c|}{\textbf{Efficiency}} & \multicolumn{2}{|c|}{\textbf{Performance}} \\ \cline{2-9}
\multicolumn{1}{l|}{} & \textbf{Run-Time} & \textbf{PlannerCalls} & \textbf{Conv.} & \textbf{Rank.} & \textbf{Run-Time} & \textbf{PlannerCalls} &\textbf{Conv.} & \textbf{Rank.} \\ \hline
\textbf{Baseline} & 105.02 & 265.05 & 21.82 & 20.24 & 194.65 & 516.57 & 16.37 & 19.54 \\ \hline
\textbf{Recompute} & 49.98 & 148.94 & 25.44 & 33.91 & 126.75 & 397.85 & 18.7 & 22.76 \\ \hline
\textbf{Prune} & 74.90 & 151.46 & 42.16 & 40.50 & 160.29 & 386.53 & 23.18 & 24.03 \\ \hline
\textbf{Both} & 36.49 & 90.68 & 42.41 & 40.21 & 97.63 & 287.36 & 20.98 & 25.82 \\ \hline
\textbf{No Recompute} & 7.10 & 9.00 & 6.77 & 9.54 & - & - & - & - \\ \hline
\end{tabular}
\caption{Comparison of all approaches across scattered and clustered goal scenarios.}
\label{tab:allResults}
\end{table*}

\paragraph{Efficiency }

Table~\ref{tab:allResults}, 
column 1, displays the average of the results of each
approach as the mean of total planner run-time 
measured
in seconds. 
 When only calling the planner once in the recognition process the \emph{No Recompute} approach takes an average of only 7.1 sec. 
 
 \emph{Baseline} has a time average of 105 sec. 

 The \emph{Pruning} heuristic 
 reduces the average time  to only 74.9 sec. And the \emph{Recompute} heuristic further reduces the average time to 49.9 sec. When utilizing both heuristics we achieved a reduction to 36.5 sec. an improvement of a substantial 

 65.25\% from the \emph{Baseline} approach.

The second column displays the average of the
results in terms of number of calls made by the recognizer to the planner.
The \emph{No Recompute} approach had an average of an extremely efficient 9 calls, i.e. the number of goals. \emph{Baseline} had an average of 265.05 calls
while the \emph{Recomputation} and \emph{Pruning} heuristics had similar success with a  further reduction to 148.9 and 151.4 calls each. Using both heuristics the number of calls was reduced to an average of only 90.6 calls, a reduction of 63.3\% from the \emph{Baseline} approach.

In conclusion we see that employing the heuristics makes a big impact on
run-time and successfully reduces overall number of calls to the planner. While the \emph{recomputation} heuristic outperformed the \emph{pruning} heuristic, both in run-time and overall number of calls, utilizing both heuristics can reduce both run-time and number of calls made to the planner by over 60\% 
from the baseline approach. The most efficient method proved to be the \emph{No Recompute} approach, only calculating $|G|$ plans. We will now show that this improvement in efficiency costs considerably in performance. 

\paragraph{Performance }

Table~\ref{tab:allResults}, column 3,  measures the average convergence to the correct result percent, higher values are better.
As we can see with no reuse of the planner at all, \emph{No Recompute} only produces 6.7\% convergence. 
As this approach does not make use of the incrementally revealed observations within the recognition process, any deviation from the initially calculated path, $i_g$, will have considerable impact on recognition results.

By converting to the online \emph{Baseline} algorithm, we were able to more than double the convergence percent to 21.8\%. 
Each incremental observation was now taken into account during the reuse of the planner and therefore had greater weight on the ranking of the goals.  
Applying both  the \emph{Pruning} and \emph{Recomputation} heuristics further improve the overall convergence. By eliminating goals the ranking process now proved to be easier, as there were less goals to compare to. Furthermore, the early elimination of goals in the pruning process was able to also eliminate the further noise these goals might introduce to the ranking process, when their paths deviated from the optimal. The \emph{Recomputation} heuristic increases it to 25.4\% and the \emph{Pruning} to  42.2\%, an improvement of 20.4\% from the \emph{Baseline} approach. When utilizing both heuristics we see that the high convergence level obtained by the \emph{Pruning} heuristic is maintained.

Column 4, measures the percent of times the correct goal was ranked first. 
Here too a higher value is better and will reflect on overall reliability of the ranking procedure. The results mostly agree with the convergence results. With no planner reuse at all, \emph{No Recompute}, 
performs poorly with a low 9.5\%. 
\emph{Baseline} more than doubles the success here as well, to 20.2\%. The \emph{Recomputation} heuristic achieves 33.9\% and the \emph{Pruning} heuristic  increases the results to 40.5\%, an improvement of 20.3\% from the \emph{Baseline} approach. Again, when applying both heuristics the success level of the \emph{Pruning} method is obtained.

Employing the heuristics has made a big impact on overall performance successfully increasing convergence and overall correct rankings. The \emph{Pruning} heuristic outperformed the \emph{Recomputation} heuristic in both measures and a combination of both heuristics maintains the high success rate leading to an improvement of over 20\% in both measures.

\vspace{-2pt}
\subsection{Sensitivity to recognition difficulty}
\label{sec:exp:harder}
\vspace{-2pt}

In online, continuous domains, the hardness of the recognition problem could possibly effect recognizer performance and efficiency. We wanted to evaluate the sensitivity of the
results shown above to the hardness of the recognition problems. We therefore added another 9 goal points 
(e.g., 19 potential goals in each recognition problem), 
for a total of 380 recognition problems.
These extra points were specifically added in close proximity to some of the preexisting 10 points, such that navigating towards any one of them appears (to human eyes) to be just as possible as any other.

Table~\ref{tab:allResults}, columns 5--6, examines the \emph{efficiency} of the different online recognition approaches over the \emph{harder} clustered goals problems. We omitted the \emph{No Recompute} heuristic in these instances as the behavior of this heuristic is very straightforward. The results are consistent with the results from the original scenario. The \emph{Baseline} approach is the least efficient, having a higher run-time and larger number of calls to the planner, than the rest. The most efficient approach is still the approach of utilizing both the \emph{Pruning} heuristic and the \emph{Recompute} heuristic together. In run-time the \emph{Recompute} heuristic is still more efficient than the \emph{Pruning} however for the measure of number of calls made to the planner we see that, for more clustered goals scenarios, the \emph{Pruning} heuristic slightly outperforms the \emph{Recompute} heuristic.

Table~\ref{tab:allResults}, columns 7--8, examines the \emph{performance} of the different online recognition approaches over the \emph{harder} clustered goals problems. For the harder problems the best performance achieved, in terms of convergence, was by the \emph{Pruning} heuristic with a convergence of 23.18\% from the end. In terms of the amount of times the correct goal was ranked first the \emph{Both} approach, combining both \emph{Pruning} and \emph{Recompute} heuristics, only slightly outperformed the \emph{Pruning} approach. The worst performance was achieved by the \emph{Baseline} approach, in terms of both criteria measured; convergence and ranked first, in congruence with the performance results of the scattered goal scenario.

\begin{table*}[htbp]
\fontsize{7.5}{9}\selectfont
\centering
\begin{tabular}{|c|c|c|c|c|}
\cline{2-5}
\multicolumn{1}{l}{} & \multicolumn{4}{|c|}{\textbf{Deterioration}} \\ \cline{2-5}
\multicolumn{1}{l|}{}  & \multicolumn{2}{|c|}{\textbf{Efficiency}} & \multicolumn{2}{|c|}{\textbf{Performance}} \\ \cline{2-5}
\multicolumn{1}{l|}{} & \textbf{Run-Time} & \textbf{PlannerCalls} & \textbf{Conv.} & \textbf{Rank.} \\ \hline
\textbf{Baseline} & 85.35\% & 94.90\% & 24.98\% & 3.46\% \\ \hline
\textbf{Recompute} & 153.58\% & 167.11\% & 26.49\% & 32.88\% \\ \hline
\textbf{Prune} & 114.01\% & 155.20\% & 45.02\% & 40.67\% \\ \hline
\textbf{Both} & 167.58\% & 216.90\% & 50.53\% & 35.79\% \\ \hline
\end{tabular}
\caption{Deterioration of performance and efficiency between scattered and clustered goal scenarios.}
\label{tab:deterioration}
\end{table*}

Table~\ref{tab:deterioration} measures the deterioration in efficiency and performance with comparison to the scattered goal scenario. The deterioration is measured in terms of deterioration percent, hence a 100\% deterioration in run-time means the planner took twice as long on average, on the harder problems. Therefore lower values are better. In terms of efficiency, we can clearly see that the least deterioration, both in run-time and number of calls to the planner, occurred for the \emph{Baseline} approach proving this approach to be very reliable with a deterioration of 85.35\% and 94.90\% respectively. 
The biggest deterioration in terms of run-time occurred for the combination of both heuristics with a deterioration of 167.58\%. This was considerably caused by the substantial deterioration of the \emph{Recompute} approach which deteriorated by 153.58\%. The \emph{Pruning} heuristic deteriorated considerable less in terms of run-time with only 114\% deterioration. 

In terms of number of calls made to the planner, again, the worst deterioration occurred for the \emph{Both} approach, with a deterioration of 216.9\% while the deterioration for each of the heuristics was considerably less; 155.2\% for the \emph{Pruning}  heuristic and 153.6\% for the \emph{Recomputation} heuristic. 

In terms of performance deterioration we again see that the most resilient approach in terms of performance, as well as efficiency, proved to be the \emph{Baseline} both in terms of Convergence and Ranked first with a deterioration of 25.98\% in convergence and only 3.46\% in ranked first. The biggest deterioration in convergence occurred for the \emph{Both} approach, as was in the efficiency results. However, in terms of ranked first the biggest deterioration occurred for the \emph{Pruning} heuristic. This was, in part, due to the fact that clustered goals make the pruning process considerably less efficient as the goals are too close to be pruned.

\subsection{Online Goal Recognition on Robots}
\label{sec:bestWorstCaseCompare}

As a final set of experiments, and to show the applicability of our approach, we implemented Alg.~\ref{alg:online} in a cooperative robotic team task.  
We used ROS~\cite{quigley2009ros} to control simulated robots in Gazebo, using the default ROS motion planner, \emph{move\_base}, in the recognition process. 
We simulated a soccer field, with two robots operating as members of the same team (Figure~\ref{fig:SoccerSimulation}). 
The observed robot was given an initial goal to travel to, proceeding to execute the plan in a straightforward manner, and the observing robot had to strategically place itself in a pre-chosen position to assist the other robot team member. If the observed robot navigated to \emph{goal 4} the strategic place to assist it on the offense would be to navigate to \emph{goal 3} and vice versa. Likewise also with goals 1 and 2. 

The observed robot always started at the same initial point in the middle of the field, while we experimented with 3 different starting points for the observing robot; two points behind the observed robots position and on parallel sides (Figure~\ref{fig:SoccerSimulation}, init points 1 and 2) 
and one point past the observed robot in the middle of the field (init point 3).
We ran 10--20 runs from each initial position to each of the goals for a total of 193 problems.

We compared our online goal recognizer (\emph{OGR}) in its baseline form, to two different approaches: (a) giving full knowledge (\emph{FK}) of the intended goal to the observing robot, ahead of time, allowing the observing robot to  navigate directly towards it; and (b) giving no (zero) knowledge (\emph{ZK}) of intended goal, thus forcing the observing robot to wait for its team member to reach its desired goal, before it can navigate towards the complementary location. 

To evaluate the different approaches we measured the overall time 
(in seconds) the simulated robot ran until reaching its target goal. The lower the time the more efficient the robot.
The results are displayed in Table~\ref{tab:ZK_FK_GM}.

\begin{figure}[htbp]
\centering{
\includegraphics[clip,width=0.60\columnwidth,keepaspectratio]{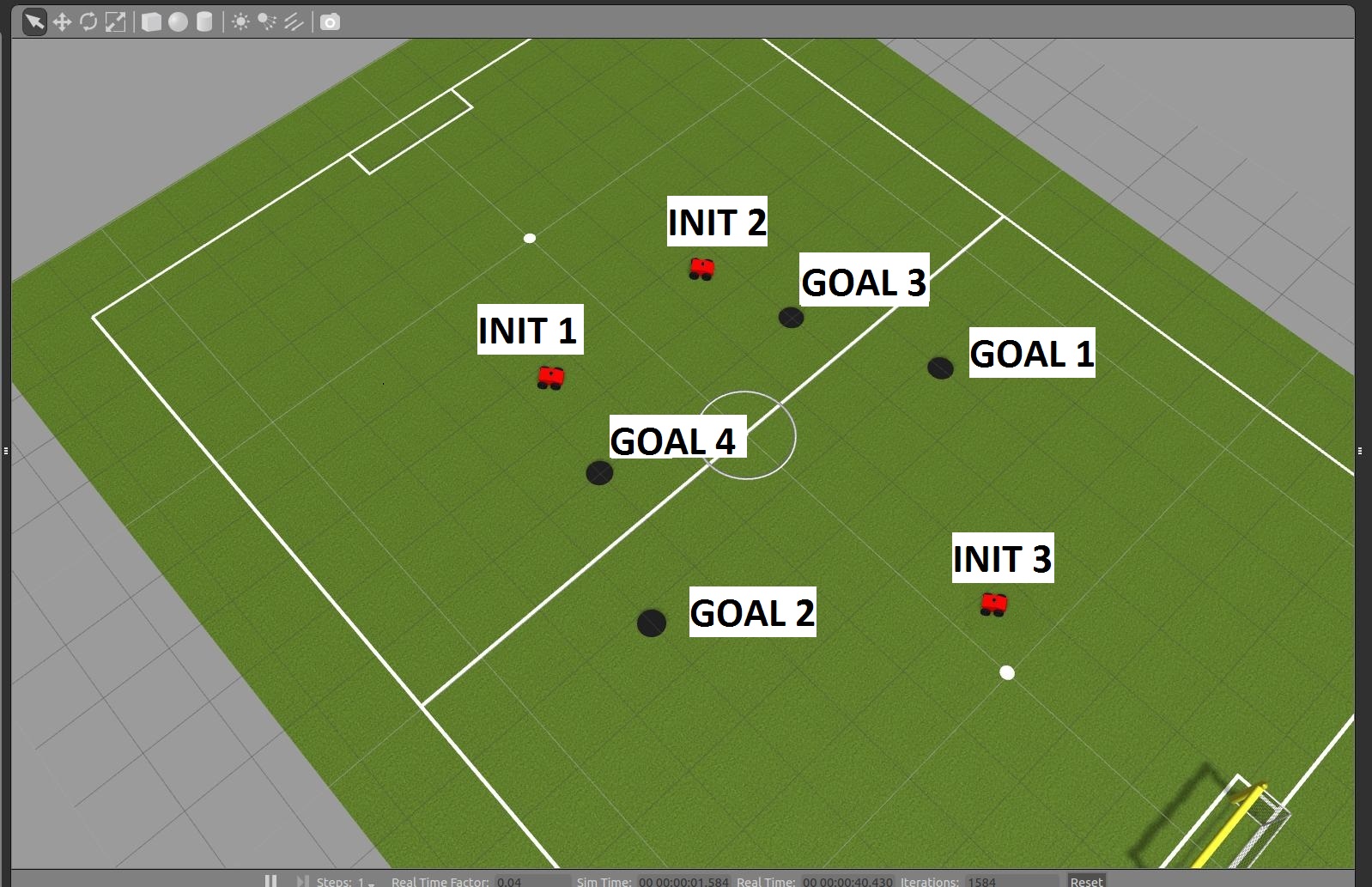}}
\caption{Experiment setup (via RVIZ)}
\label{fig:SoccerSimulation}
\end{figure}

\begin{table}[htbp]
\fontsize{6.5}{9}\selectfont
\centering
\begin{tabular}{|c|c|c|c|c|}
\cline{3-5}
\multicolumn{1}{l}{} & \multicolumn{1}{l|}{} & \textbf{FK} & \textbf{OGR} & \textbf{ZK} \\ \hline
\textbf{I1} & \textbf{G1} & 10.00 & 17.35 & 21.50 \\ \hline
\textbf{} & \textbf{G2} & 5.80 & 12.31 & 17.18 \\ \hline
\textbf{} & \textbf{G3} & 9.10 & 16.19 & 20.43 \\ \hline
\textbf{} & \textbf{G4} & 5.80 & 17.10 & 26.03 \\ \hline
\textbf{I2} & \textbf{G1} & 5.80 & 14.09 & 17.67 \\ \hline
\textbf{} & \textbf{G2} & 9.10 & 19.56 & 24.45 \\ \hline
\textbf{} & \textbf{G3} & 10.00 & 15.89 & 25.45 \\ \hline
\textbf{} & \textbf{G4} & 12.35 & 15.32 & 32.62 \\ \hline
\textbf{I3} & \textbf{G1} & 12.41 & 17.36 & 20.88 \\ \hline
\textbf{} & \textbf{G2} & 10.10 & 18.65 & 24.26 \\ \hline
\textbf{} & \textbf{G3} & 9.24 & 10.62 & 20.66 \\ \hline
\textbf{} & \textbf{G4} & 5.72 & 13.40 & 20.40 \\ \hline
\end{tabular}
\caption{Online goal recognizer  vs. full and zero knowledge}
\label{tab:ZK_FK_GM}
\end{table}

The results show that 
the goal recognition approach substantially improves on the \emph{zero knowledge} approach, while requiring no precalculations; all needed plans are generated on-the-fly via the planner. Understandably our approach falls short of the \emph{full knowledge} approach as it generates hypotheses on the fly following observations, which leads to some deviations from the optimal, direct route.

\section{Summary}
\label{sec:fin}
\vspace{-4pt}
We presented 
an efficient, heuristic, online goal recognition
approach which utilizes a planner in the recognition process to generate recognition hypotheses. We identified key decision points which effect both overall run-time and the number of calls made to the planner and introduced a generic online goal recognition algorithm along with two heuristics to improve planner performance and efficiency in navigation goal recognition. 
We evaluated the approach in a challenging navigational goals domain over hundreds of experiments and varying levels of problem complexity.
The results demonstrate the power of our proposed heuristics and show that,
while powerful by themselves, a combination of them leads to a reduction of a substantial
63\% of the calls the recognizer makes to the planner and planner run-time in comparison with previous work. This, while showing an increase of over 20\% in recognition measures. We further demonstrated the algorithm in a realistic simulation of a simple robotic team task, and showed that it is capable of recognizing goals using standard
robotics motion planners.

%\clearpage

\bibliographystyle{named}
\bibliography{full,all,plan-rec,myconf,myjournals,AComputationalCognitiveModelOfMirroringProcessesBibliography}

\end{document}